\def\year{2021}\relax
\documentclass[letterpaper]{article} 
\usepackage{aaai21}  
\usepackage{times}  
\usepackage{helvet} 
\usepackage{courier}  
\usepackage[hyphens]{url}  
\usepackage{graphicx} 
\usepackage{graphicx}  
\usepackage{natbib}  
\usepackage{caption} 
\usepackage{amsmath,amssymb}

\newtheorem{definition}{Definition}
\newtheorem{lemma}{Lemma}
\newtheorem{theorem}{Theorem}
\newtheorem{claim}{Claim}
\newcommand{\K}{{\sf K}}
\newcommand{\W}{{\sf W}}
\renewcommand{\S}{{\sf S}}
\newcommand{\A}{{\sf A}}
\renewcommand{\phi}{\varphi}
\newcommand{\simS}{\stackrel{{\sf S}}{\sim}}
\newcommand{\simA}{\stackrel{{\sf A}}{\sim}}

\newenvironment{proof}{\noindent{\sc Proof.}}{\hfill $\boxtimes\hspace{2mm}$\linebreak}
\newcommand{\qed}{\hfill $\boxtimes\hspace{1mm}$}
\newenvironment{proof-of-claim}{\noindent{\sc Proof of Claim.}}{\hfill $\boxtimes\hspace{2mm}$\linebreak}

\begin{document}

\title{Logic of Identification}
\title{Logic of Know-Who}
\title{Know-Who: Logic of Identities}
\title{Nobody Knows Who Knows: The Logic of Know-Who}
\title{Knowing Who Knows: The Logic of Know-Who}
\title{Epistemic Logic of Know-Who}

\author {
    Sophia Epstein,\textsuperscript{\rm 1}
    Pavel Naumov\textsuperscript{\rm 2}\\
}
\affiliations {
    \textsuperscript{\rm 1} Claremont McKenna College \\
    \textsuperscript{\rm 2} King's College \\
    sepstein22@cmc.edu, pgn2@cornell.edu
}

\maketitle

\begin{abstract}

The paper suggests a definition of ``know who'' as a modality using Grove-Halpern semantics of names. It also introduces a logical system that describes the interplay between modalities ``knows who'', ``knows'', and ``for all agents''. The main technical result is a completeness theorem for the proposed system.
\end{abstract}

\section{Introduction}

The ability of artificial agents to properly identify humans and other machines  is critical in many AI applications from online and checkout-less shopping, robotic nurses, and unmanned aircraft systems to security, law-enforcement, and lethal autonomous weaponry. Most of the current systems rely on physical identifiers such as facial images, fingerprints, signatures, government-issued IDs, iris recognition, credit card security chips, passwords, and radio signals. Knowing one of these identifiers does not imply knowing the others or knowing who the person (or a machine) ``really'' is. In this paper we propose a formal framework for defining and reasoning about ``knowing (somebody) who''.

\subsection{The Night Stalker}

On July 27th, 1981, a real estate agent came to see a house for sale near Santa Barbara, California. Inside the house at 449 Toltec Way in Goleta, the agent found the bodies of 35-year-old Cheri Domingo, who was house-sitting the place, and of her former boyfriend, 27-year-old Gregory Sanchez~\cite{h81sbnews}. Within several days, the Santa Barbara County sheriff's spokesman Russ Birchim announced that the police knew who the killer was. He was the same man who committed a nonfatal knife attack on another couple in the same neighborhood 22 months ago. Birchim said that the deputies dubbed the killer ``Night Stalker''~\cite{h81latimes}.

Did Birchim really {\em know} who the murderer was? It took almost 40 years for the police to find out that ``Night Stalker'' is actually ``East Area Rapist'' who raped 50 people in Northern California in the 1970s, almost 400 miles away from Santa Barbara. The same person was also known as ``Visalia Ransacker'' and ``Golden State Killer''. It also was discovered that the same person was known to California police as sergeant DeAngelo serving in Auburn, California police forces from August 1976 to July 1979, when he was arrested and sentenced to six months probation for shoplifting a hammer and dog repellent~\cite{l18sacbee,so18latimes}. Did the sentencing judge know who the shoplifter {\em really} was? 

As the example above shows, the same person might be known under different names and knowing one of the person's names does not necessarily imply knowing all of them. To define the meaning of ``know who'' one needs to fix a name space. Knowing the person under one name space does not imply knowing the same person under another. For example, Birchim knew who the murderer was using a hypothetical name space consisting of ``Night Stalker'', ``Morning Stalker'', ``Day Stalker'', and ``Evening Stalker'', but did not know the murderer in a hypothetical space ``East Area Rapist'', ``North Area Rapist'', ``West Area Rapist'', and ``South Area Rapist''. 

There are many other real-world situations with multiple name spaces.  Knowing an author under a pen name might not mean knowing the author's birth name. Knowing students by face is very different then knowing their names or ID numbers. Children separated at birth might know each other, but not know that they are related.

Aloni refers to name spaces as ``conceptual covers''~(\citeyear{a05jpl,a18hintikka}). In this paper, we  describe the universal properties of ``know-who"' that are true for any fixed name space.

\subsection{Outline}
The rest of the paper is organized as follows. In the next section we introduce  and discuss Grove-Halpern  epistemic models with names. Then, we describe syntax of our logical system and give its formal semantics.  After this we highlight a possible extension of our logic by explicit names, discuss connection between our semantics and de dicto/de re knowledge, and review the related literature.
In the next two sections, we introduce the axioms of the Logic of Know-Who and prove their soundness.  Section Completeness Overview highlights the key steps in the proof of the completeness. The actual proof of the completeness is given in
the full version of this paper~\cite{en20arxiv}. The last section concludes. 

\section{Grove-Halpern Models}\label{epistemic model section}

The formal semantics of names that we use in this paper was first proposed by Grove and Halpern to study  modality ``for all agents with a given name"~(\citeyear{gh91kr}).

\begin{definition}\label{model}
A tuple $(S,A,P,\{\sim_a\}_{a\in A},N,I,\pi)$ is called a model if
\begin{enumerate}
    \item $S$ is an arbitrary set of ``states'',
    \item $A$ is an arbitrary set of ``agents'',
    \item $P$ is a function that maps each agent $a\in A$ into a set of states $P(a)\subseteq S$ in which the agent is ``present'',
    \item $\sim_a$ is an ``indistinguishability'' equivalence relation on set $P(a)$ for each agent $a\in A$,
    \item $N$ is a set of ``names'',
    \item $I\subseteq A\times S\times N \times A$ is an ``identification mechanism'' relation satisfying the following two conditions: 
    \begin{enumerate}
        \item for each $a\in A$, each  $s\in P(a)$, and each $n\in N$, there is at least one agent $a'\in A$ such that $(a,s,n,a')\in I$,
        \item for each $a\in A$, each $s\in P(a)$, each  $n\in N$, and each agent $a'\in A$, if $(a,s,n,a')\in I$, then $s\in P(a')$,
    \end{enumerate}
    
    \item for each propositional variable $p$, set $\pi(p)$ is an arbitrary set of pairs $(a,s)$ such that $a\in A$ and $s\in P(a)$.
\end{enumerate}
\end{definition}

\begin{figure}
\begin{center}
\vspace{0mm}
\scalebox{0.45}{\includegraphics{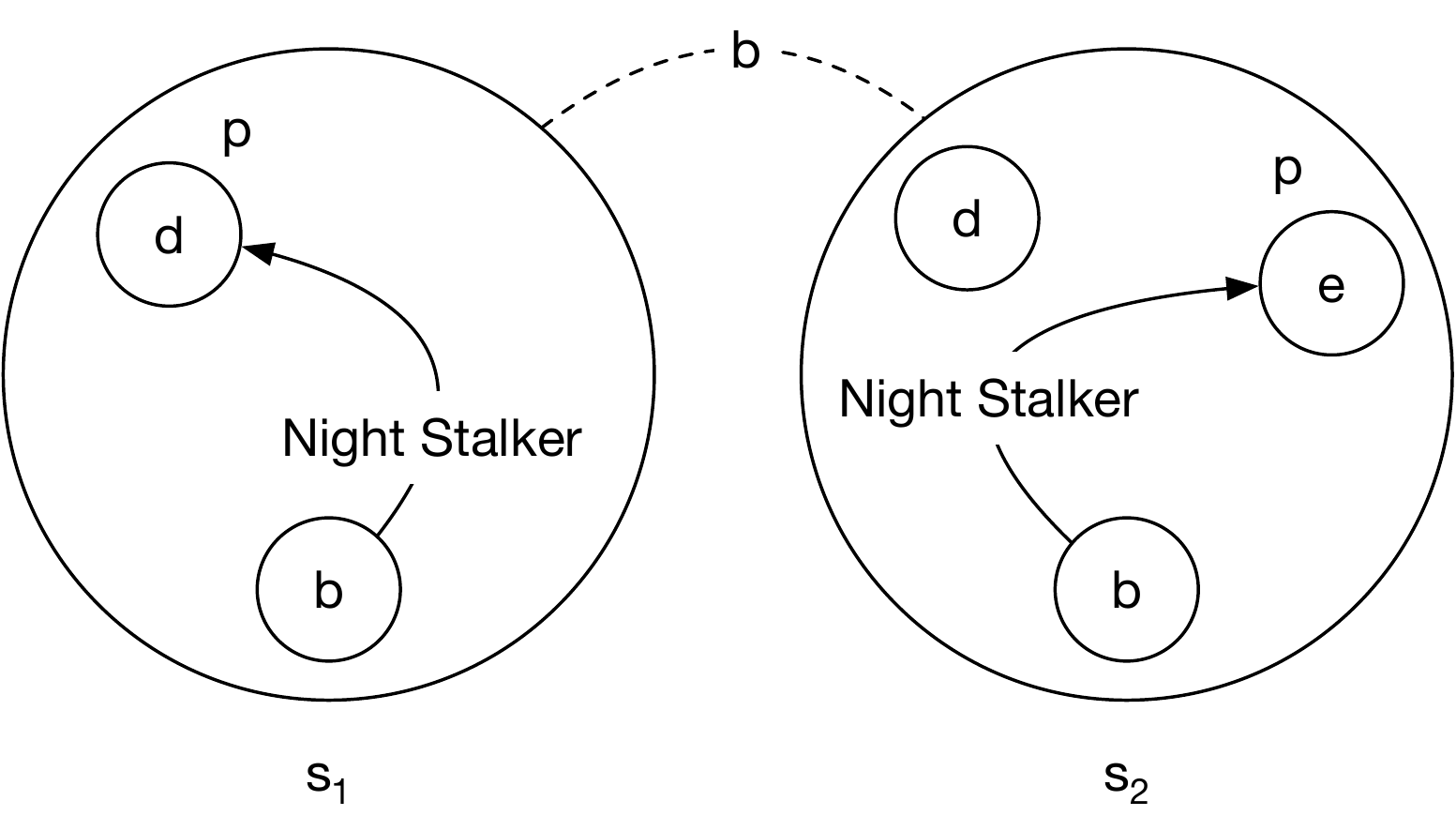}}
\vspace{-1mm}
\caption{The Night Stalker Model (not all edges are shown). Propositional variable $p$ means ``is the murderer''.}\label{night stalker figure}
\end{center}
\end{figure}
Figure~\ref{night stalker figure} depicts a Grove-Halpern model for the Night Stalker example.
Grove-Halpern models use states and indistinguishability relation $\sim_a$ to capture knowledge in almost the same way as it is done in Kripke models for the epistemic logic S5. The diagram in Figure~\ref{night stalker figure} depicts two states, $s_1$ and $s_2$, indistinguishable by Birchim ($b$). In state $s_1$, DeAngelo ($d$) is the murderer. In state $s_2$, somebody else,  agent $e$ is the murderer. The significant difference between S5 models and Grove-Halpern models is that the latter do not assume that each agent is present in each state. This generalization of semantics would be insignificant in standard epistemic logic, but it is important for our logical system because its language contains modality ``for all agents in the given state''. To capture which agent is present in which state, in addition to the set of states $S$ and the set of agents $A$, the model also includes set $P(a)\subseteq S$ for each agent $a\in A$. Set $P(a)$ is the set of states in which an agent $a$ is ``present''. In the Night Stalker model, agents $b$ and $d$ are present in both states and agent $e$ is only present in state $s_2$, see Figure~\ref{night stalker figure}. Thus, $P(b)=P(d)=\{s_1,s_2\}$ and $P(e)=\{s_2\}$. Intuitively, an agent cannot distinguish or not distinguish states in which she is not present. Thus, we assume that the indistinguishability relation $\sim_a$ is only defined on the set of states $P(a)$ in which the agent $a$ is present.

As we have seen in our introductory example, the meaning of {\em know-who} is impossible to define without specifying the name space. Any Grove-Halpern model assumes a fixed set of names $N$. In our running example, $N=\{\mbox{``Night Stalker''}\}$.  The identification mechanism $I$ is the key part of defining a name space. This mechanism specifies which name could be used to refer to which agent. Grove-Halpern models take one of the most general approaches of assigning names to agents. They allow names like ``my mother'' that might refer to different women when used by different people. Thus, the meaning of a name is assumed to be {\em agent-specific}. They also allow names like ``my best friend'' that might refer to different people in different states. Thus, the meaning of a name is assumed to be {\em state-specific}. Furthermore, it is assumed that an agent might use different names to refer to the same person. Hence, for example, there could be a name space that simultaneously includes names ``Night Stalker'' and ``East  Area  Rapist'' for the same person. In such a name space, just like in our example, spokesman Birchim would be able to claim that he knows who the killer is even if he only can identify the perpetrator as ``Night Stalker'' but not as ``East  Area  Rapist''. Finally, the models allow names like ``my parent'' that the same person in the same state might use to refer to two different people. If one says that she knows who, her parent, raised her, then we interpret this as her saying that she knows that she was raised by both parents. To support all these features,  an identification mechanism $I$ is specified as a set of tuples $(a,s,n,a')\in A\times S\times N\times A$. If $(a,s,n,a')\in I$, then agent $a$ in state $s$ might use name $n$ to refer to agent $a'$. The mechanism of the Night Stalker model is depicted by the directed edge on the diagram in Figure~\ref{night stalker figure}. For instance, the directed edge labeled with name Night Stalker from agent $b$ to agent $d$ inside state $s_1$ means that $(b,s_1,\mbox{Night Stalker},d)\in I$. In other words, name Night Stalker refers to agent $d$ when used by agent $b$ in state $s_1$. 

We believe that our work could be relatively easily generalized to a setting with multiple name spaces similar to one used in~\cite{a05jpl,a18hintikka}. If multiple name spaces would be present in the semantics, then the logical system could have multiple know-who modalities labeled by name spaces. Generally speaking, these modalities will be unrelated. In other words, knowing who in one name space does not say anything about knowing who in the other. In this paper, we restrict consideration to a single name space.


In spite of allowing very general identification mechanisms, we impose on them two restrictions captured by conditions 6(a) and 6(b) of Definition~\ref{model}. The first of these conditions states that for any agent $a$, any state $s\in P(a)$, and any name $n\in N$, there must exist at least one   agent $a'$ that agent $a$ refers to by name $n$ in state $s$. In other words, we want to exclude cases when Birchim would claim that he knows that, say, the Santa Claus is the murderer, when there is no single person who is Santa Claus.  The second condition requires that any of the above agents $a'$ must be present in state $s$. This condition guarantees that  ``Night Stalker'' exists in the epistemic state in which Birchim knows that ``Night Stalker'' is the murderer. We introduce these two conditions on name spaces because we believe that without them our formal definition of ``know-who'' modality, see Definition~\ref{sat}, does not reflect the informal meaning of ``knowing who''.  

Another important difference between Grove-Halpern models and the standard Kripke semantics for epistemic logic S5 is that valuation function $\pi$ maps propositional variables not into sets of states, but into sets of pairs $(a,s)$ consisting of an agent $a$ and a state $s\in P(a)$ in which agent $a$ is present. In other words, propositional variables interpreted  as sentences in which the subject is omitted. Grove and Halpern call them {\em relative} sentences. In our example from Figure~\ref{night stalker figure}, the phrase ``is the murderer'' from the sentence
``Spokesman Russ Birchim knows who is the murderer,'' is the meaning of proposition $p$. Set $\pi(p)$ is the set of all pairs $(a,s)$ such that statement $p$ is true about agent $a$ in state $s$. 

Grove and Halpern~(\citeyear{gh91kr}) first introduce Definition~\ref{model} without conditions 6(a) and 6(b). Later they add condition 6(b) but simultaneously make names no longer agent-specific~(\citeyear{gh93jlc}). In his third work, Grove again makes names agent-specific and adds condition 6(a), but in a form stronger than ours: ``exactly one'' instead of ``at least one''~(\citeyear{g95ai}). The notion of a conceptual cover~\cite{a05jpl} is significantly more restrictive. It requires each agent to have a unique name and each name to refer to a unique agent.

\section{Syntax}\label{syntax section}

Not only we interpret propositional variables as relative sentences, but we do the same with all modal formulae in our language. In Definition~\ref{sat}, we will specify formal semantics of our logic as a ternary relation $(a,s)\Vdash\phi$ between an agent, a states, and a formula. Informally, it means that formula $\phi$ is true in state $s$ {\em about} agent $a$. In our introductory example,
$
(a,s)\Vdash \mbox{``is the murderer''}
$
where $a$ is the person who was known as ``Night Stalker''. This approach allows a very straightforward treatment of know-who modality $\W$. Namely, to state that spokesman Birchim knows who is the murderer, we write
$$
(b,s)\Vdash \W(\mbox{``is the murderer''}),
$$
where $b$ is the person known as spokesman Birchim. Imagine a hypothetical situation when police announces a press conference at which Birchim will disclose the name of the murderer. Before the conference starts, any journalist $j$ attending the conference, would not know yet who is the murderer:
$$
(j,s)\Vdash \neg\W(\mbox{``is the murderer''}),
$$
but the journalist would know who, spokesman Birchim, knows who is the murderer:
$$
(j,s)\Vdash \W\W(\mbox{``is the murderer''}).
$$
We treat knowledge modality $\K$ in a similar subscript-free fashion. Namely, we write $(a,s)\Vdash\K\phi$ if in a state $s$ an agent $a$ knows that statement $\phi$ is true {\em about} the agent $a$. For example, because ``Night Stalker'' knows that he himself is the murderer,
$$
(a,s)\Vdash \K(\mbox{``is the murderer''}).
$$
where $a$ is the person who was known as ``Night Stalker''. In addition to modalities for know-who $\W$ and knowledge $\K$, our system also includes modality $\A$ that stands for ``all agents in the state''. For example, the journalist would know that not all people are innocent:
$$
(j,s)\Vdash \K\neg\A(\mbox{``is not the murderer''}).
$$
Although relative sentences have already been used by Grove and Halpern~(\citeyear{gh91kr}), subscript-free modalities were introduced much later in Friendship Logic~\cite{slg13tark} that contains modalities $\K$, $\A$, and ${\sf F}$. The latter stands for ``for all my friends''.

In this paper we propose a sound and complete logical system that describes the interplay between modalities $\W$, $\K$, and $\A$.
We assume a fixed countable set of propositional variables. The language $\Phi$ of our system is defined by the grammar:
$$
\phi:=p\;|\;\neg\phi\;|\;\phi\to\phi\;|\;\W\phi\;|\;\K\phi\;|\;\A\phi.
$$
We read $\W\phi$ as ``knows an agent for whom $\phi$ is true'', $\K\phi$ as ``knows that $\phi$ is true about herself'', and $\A\phi$ as ``$\phi$ is true for all agents''.
We suppose that Boolean constant $\top$ and conjunction $\wedge$ are defined in the standard way. For any finite set $X\subseteq\Phi$, by $\wedge X$ we mean the conjunction of all formulae in $X$. By definition, $\wedge\varnothing$ is $\top$.

\section{Semantics}

Next we define formal semantics of our logical system. The key part of this definition is item 6 that specifies the meaning of know-who modality $\W$.

\begin{definition}\label{sat}
For any model $(S,A,P,\{\sim_a\}_{a\in A},N,I,\pi)$, any agent $a\in A$, any state $s\in P(a)$, and any formula $\phi\in \Phi$, satisfiability relation $(a,s)\Vdash\phi$ defined as follows:
\begin{enumerate}
    \item $(a,s)\Vdash p$ if $(a,s)\in \pi(p)$,
    \item $(a,s)\Vdash\neg\phi$ if $(a,s)\nVdash\phi$,
    \item $(a,s)\Vdash\phi\to\psi$ if $(a,s)\nVdash\phi$ or $(a,s)\Vdash\psi$,
    \item $(a,s)\Vdash \A\phi$ if $(a' ,s)\Vdash\phi$ for each  agent $a'\in A$ such that $s\in P(a')$,
    \item $(a,s)\Vdash \K\phi$ if $(a,s')\Vdash\phi$ for each  state $s'\in P(a)$ such that $s\sim_a s'$,
    \item $(a,s)\Vdash \W\phi$ when there is a name $n\in N$ such that for each state $s'\in P(a)$ and each agent $a'\in A$, if $s\sim_a s'$ and $(a, s',n,a')\in I$, then $(a',s')\Vdash \phi$.
\end{enumerate}
\end{definition}



State $s'$ in item 6 is used to capture the ``know'' part of ``know-who''. Namely, we require that the same name $n$ identifies the right person in all states $s'$ that agent $a$ cannot distinguish from the current state $s$. This is very similar to how the modality {\em know-how} is often defined in the literature~\cite{aa16jlc,fhlw17ijcai,nt17aamas,nt18ai,nt18aamas,nt18aaai}.

In spite of its generality, our definition of know-who has limitations. Namely, it does not support the case when ``who'' is a group of agents as in  ``John knows who is conspiring against whom'' and  ``John knows who insulted whom in whose presence''~\cite{bl03}. The complexity of these settings comes not from the fact that know-who refers to a set of agents, but rather from the fact that this set has a structure. An ``insulter'' is different from the ``insultee'' and the ``observer''. A hypothetical group know-who modality would need not only refer to a name of the group, but also to specify {\em who is who} in this group.

\section{Explicit Names}

In the standard epistemic logic, only state $s$ is placed on the left-hand-side of the satisfiability relation $\Vdash$. As a result, statements in this logic are about states, not agents. By following~\cite{gh91kr,gh93jlc,g95ai,slg13tark} and placing both the state and the agent on the left-hand-side of $\Vdash$, we gain the ability to express statements about states, statements about agents, and statements about agents in states. 

It appears, however, that we lose the ability to express statements like ``in state $s$ agent $a$ knows that agent $b$ knows $\phi$'', which is expressible in the standard epistemic logic by $s\Vdash \K_a\K_b\phi$. This ability could be easily restored by adding {\em reference by name} modality $@_n$ to our language to form language $\Phi^@$:
$$
\phi:=p\;|\;\neg\phi\;|\;\phi\to\phi\;|\;\W\phi\;|\;\K\phi\;|\;\A\phi\;|\;@_n\phi,
$$
where $n$ is any name. We read $@_n$ as ``for any agent with name $n$''. The semantics of language $\Phi^@$ could be defined using Grove-Halpern models by adding the following part to Definition~\ref{sat}:

\begin{definition}
$(a,s)\Vdash @_n\phi$  when for each agent $a'\in A$, if  $(a, s,n,a')\in I$, then $(a',s)\Vdash \phi$.
\end{definition}

Using modality $@$, statement ``in state $s$ agent $a$ knows that agent $b$ knows $\phi$'' could be written in our system as $(a,s)\Vdash\K@_{\mbox{\footnotesize{Bob}}}\,\K\phi$, assuming that in state $s$ agent $a$ refers to agent $b$ as Bob. Note that with this addition, we still retain the ability to have statements about states and agents. For example, statement $\phi$ in the above example could be any statement about state $s$ and/or agent $b$. Using modality $@$ we can express the fact that the agent known to spokesman Russ Birchim as ``the Night Stalker'' is the murderer, see Figure~\ref{night stalker figure}, as
$$
(b,s_1)\Vdash @_{\mbox{\footnotesize Night Stalker}}\,p.
$$
We can also express the fact that Birchim knows this as
$$(b,s_1)\Vdash \K@_{\mbox{\footnotesize Night Stalker}}\,p.$$ 
In this paper we give a complete logical system that describes the universal properties expressible in language $\Phi$, leaving proving completeness of a similar system for language $\Phi^@$ for the future.

\section{Knowing {De Dicto} vs. {De Re}}

There has been a long tradition of discussions in philosophy whether one should distinguish knowing the name of a object from knowing the object itself. These two forms of knowledge are often referred to as {\em de dicto} and {\em de re} knowledge respectively. Here is one of the examples used in the literature to distinguish these two forms of knowledge:

\begin{quote}
    {\em Suppose, for example, that I’m asked who is Obama. While in some contexts, say at an exam at school, in order to answer it I have to know that Obama is the president of the US, in some other context, say at a party at the White House, what is needed is knowledge of someone in particular that he is Obama.}~\cite{co13sl}
\end{quote}

The authors of this example consider ``knowing who Obama is'' in the first case as a de dicto knowledge of the fact that Obama is a name of 44th President of the United States, while ``knowing who Obama is'' in the second example as de re knowledge of Obama as a physical object. We disagree. To us, the only difference between these two cases is that the first is using naming system based on job title (``44th President'') while the second is using naming system based on visual identity. To make our point about how artificial the distinction between knowing the name and knowing the object is, consider a hypothetical example when baby Barack Obama was accidentally switched with baby Omar Bari at birth in the hospital. As a result, Omar Bari grew up under name Barack Obama and became the 44th U.S. president, while ``real''  Barack Obama works as a hotel manager in Hawaii under name Omar Bari. When somebody at a party in White House is asking who is Obama, are they looking for the President or the ``real'' Obama, the manager? 

Wang and Seligman~(\citeyear{ws18aiml}) argue for the distinction between {de dicto} and {de re} knowledge using the broken robot example originally proposed in~\cite{g95ai}:
\begin{quote}
    {\em Grove gives an interesting example of a robot with a mechanical problem calling out for help (perhaps in a Matrix-like future with robots ruling the world unaided by humans). To plan further actions, the broken robot, called $a$, needs to know if its request has been heard by the maintenance robot, called $b$. But how to state exactly what $a$ needs to know?}
\end{quote}
To illustrate de re/de dicto distinction they list four different things that $a$, the broken robot, might know: 
    (i) the robot named $b$ knows that the robot named $a$ needs help,
   (ii) the robot named $b$ knows that it, i.e. the broken robot, needs help,
    (iii) the maintenance robot knows that the robot named $a$ needs help, 
    (iv) the maintenance robot knows that it, i.e. the broken robot, needs help.
Although we agree that these four sentences have different meanings, we believe that this difference could be completely captured by distinguishing name space containing names $a$ and $b$ from the name space containing names ``maintenance robot'' and ``broken robot''.

Since the distinction between de dicto and de re appears unimportant in our setting, we do not stress it in this paper.

\section{Related Literature}\label{Related Literature section}



Hintikka~(\citeyear{h62}) argues that statement ``agent $a$ knows who is agent $b$'' could be expressed in a first order epistemic logic as $\exists x\,\K_a(b=x)$. Wang agrees, stating that ``to formalize `I know who $b$ is' we do need quantifiers''~(\citeyear{w18hintikka}).
Bo\"{e}r and Lycan discuss multiple meanings of ``know-who'' in English~(\citeyear{bl03}). 
Aloni adds conceptual covers (name spaces) to modal language with first-order quantifiers and proves the completeness of such system~(\citeyear{a05jpl}). She later further develops this approach~(\citeyear{a18hintikka}). 
Wang and Seligman's related work, while not dealing directly with know-who,  proposes a sound and complete term logic capturing properties of non-rigid names that might not be common knowledge~(\citeyear{ws18aiml}). Unlike these works, we treat know-who as a single modality and avoid the use of quantifiers. 

Wang calls know-who one of ``know-wh'' types of knowledge: know-who, know-how, know-whether, know-what~(\citeyear{w18hintikka}). Among them, modal properties of {\em know-how} are studied the most~\cite{aa16jlc,fhlw17ijcai,w15lori,w17synthese,nt17aamas,nt18ai,nt18aamas,nt18aaai,cn20ai}. 
Logics of {\em know-whether} are studied in~\cite{fwv15rsl,fgksv20arxiv}.
Different forms of {\em know-value} logics are investigated in~\cite{wf13ijcai,gw16aiml,vgw17icla}. 
Xu, Wang, and Studer proposed a logic of {\em know-why}~(\citeyear{xws19synthese}).

\section{Axioms}\label{Axioms section}

In addition to propositional tautologies in language $\Phi$, our logical system has the following axioms, where here and in the rest of the paper $\Box$ is either modality $\A$ or modality $\K$:

\begin{enumerate}
    \item Truth: $\Box\phi\to\phi$,
    \item Distributivity: $\Box(\phi\to\psi)\to(\Box\phi\to\Box\psi)$,
    \item Negative Introspection: $\neg\Box\phi\to\Box\neg\Box\phi$,
    \item Know-Nobody: $\A\neg\phi\to\neg\W\phi$,
    \item Know-All: $\K\A(\phi\to\psi)\to(\W\phi\to\W\psi)$,
    \item Introspection of Know-Who: $\W\phi\to \K\W\phi$.
\end{enumerate}
The Truth, the Distributivity, and the Negative Introspection are standard S5 axioms. The Know-Nobody axiom says that if there is no agent in the current state for whom $\phi$ is true, then the current agent cannot know somebody for whom $\phi$ is true. The Know-All axiom says that if the agent knows that $\phi\to\psi$ for all agents in the current state and the current agent knows someone for whom $\phi$ is true, then she also knows someone for whom $\psi$ is true. The Introspection of Know-Who axiom says that if the current agent knows for whom $\phi$ is true, then she knows that she knows.

We write $\vdash\phi$ if formula $\phi$ is provable in our logical system using the Modus Ponens inference rule and the three forms of the Necessitation inference rule:
$$
\dfrac{\phi,\;\;\;\ \phi\to\psi}{\psi}
\hspace{10mm}
\dfrac{\phi}{\A\phi}
\hspace{10mm}
\dfrac{\phi}{\K\phi}
\hspace{10mm}
\dfrac{\phi}{\W\phi}.
$$
We write $X\vdash\phi$ if formula $\phi$ is provable from the theorems of our logical system and the set of additional formulae $X$ using only the Modus Ponens inference rule.
%

The next two lemmas state well-known facts about S5 modality.
We give their proofs in the full version of the paper.
\begin{lemma}\label{super distributivity}
If $\phi_1,\dots,\phi_n\vdash\psi$, then $\Box\phi_1,\dots,\Box\phi_n\vdash\Box\psi$.
\end{lemma}

\begin{lemma}\label{positive introspection lemma}
$\vdash \Box\phi\to\Box\Box\phi$.
\end{lemma}

\section{Soundness}\label{Soundness section}

In this section we prove the soundness of our logical system. The soundness of the Truth, the Distributivity, and the Negative Introspection axioms, as well as the Modus Ponens and the three forms of the Necessitation inference rule, is straightforward. Below we show the soundness of each remaining axiom as a separate lemma. In these lemmas we assume that $(a,s)$ is an arbitrary pair of an agent $a$ and a state $s$ such that $s\in P(a)$.

\begin{lemma}
If $(a,s)\Vdash \A\neg\phi$, then $(a,s)\nVdash \W\phi$.
\end{lemma}
\begin{proof}
Suppose that $(a,s)\Vdash \W\phi$. Thus, by item 6 of Definition~\ref{sat}, there is a name $n\in N$ such that for each state $s'\in P(a)$ and each agent $a'\in A$, if $s\sim_a s'$ and $(a, s',n,a')\in I$, then $(a',s')\Vdash \phi$. 

Note that $s\in P(a)$ by the assumption in the preamble of this section and  $s\sim_a s$ because $\sim_a$ is an equivalence relation. Thus, for each agent $a'\in A$, if  $(a,s,n,a')\in I$, then $(a',s)\Vdash \phi$. 

By condition (a) of item 6 in Definition~\ref{model}, there is at least one agent $a'\in A$ such that $(a,s,n,a')\in I$. Thus, $(a',s)\Vdash \phi$. Hence, $(a',s)\nVdash \neg\phi$ by item 2 of Definition~\ref{sat}.  Therefore, $(a,s)\nVdash \A\neg\phi$ by item 4 of Definition~\ref{sat}.
\end{proof}




\begin{lemma}
If $(a,s)\Vdash \K\A(\phi\to\psi)$ and $(a,s)\Vdash\W\phi$, then $(a,s)\Vdash\W\psi$.
\end{lemma}
\begin{proof}
Suppose that $(a,s)\Vdash\W\phi$. Thus, by item 6 of Definition~\ref{sat}, there is a name $n\in N$ such that for each state $s'\in P(a)$ and each agent $a'\in A$, if $s\sim_a s'$ and $(a, s',n,a')\in I$, then $(a',s')\Vdash \phi$.

Consider any state $s'\in P(a)$ and any agent $a'\in A$ such that $s\sim_a s'$ and $(a, s',n,a')\in I$. Then, as we have shown above,
\begin{equation}\label{Latvia}
   (a',s')\Vdash \phi. 
\end{equation}
By item 6 of Definition~\ref{sat}, it will suffice to show that $(a',s')\Vdash \psi$. Indeed, by item 5 of Definition~\ref{sat} assumption $(a,s)\Vdash \K\A(\phi\to\psi)$ implies that $(a,s')\Vdash \A(\phi\to\psi)$  because $s'\in P(a)$ and $s\sim_a s'$. 

Note that $s'\in P(a')$ by condition (b) of item 6 in Definition~\ref{model} because $(a, s',n,a')\in I$. Hence, statement $(a,s')\Vdash \A(\phi\to\psi)$ implies that $(a',s')\Vdash \phi\to\psi$ by item 4 of Definition~\ref{sat}. Therefore, $(a',s')\Vdash \psi$ by item 3 of Definition~\ref{sat} and statement~(\ref{Latvia}).
\end{proof}

\begin{lemma}
If $(a,s)\Vdash \W\phi$, then $(a,s)\Vdash\K\W\phi$.
\end{lemma}
\begin{proof}
Consider any state $s'\in P(a)$ such that $s\sim_a s'$. By item 5 of Definition~\ref{sat}, it suffices to show that $(a,s')\Vdash\W\phi$.

By item 6 of Definition~\ref{sat}, assumption $(a,s)\Vdash \W\phi$ implies that there is a name $n\in N$ such that for each state $s''\in P(a)$ and each agent $a'\in A$, if $s\sim_a s''$ and $(a, s'',n,a')\in I$, then $(a',s'')\Vdash \phi$. Recall that $s\sim_a s'$. Thus, for each state $s''\in P(a)$ and each agent $a'\in A$, if $s'\sim_a s''$ and $(a, s'',n,a')\in I$, then $(a',s'')\Vdash \phi$ because $\sim_a$ is an equivalence relation. Therefore, $(a,s')\Vdash\W\phi$ by item 6 of Definition~\ref{sat}. 
\end{proof}

\section{Completeness Overview}\label{Completeness Overview section}

In this section we highlight the key steps in the proof of the completeness. The prove itself is located in the full paper. 

In modal logic, a proof of a completeness usually constructs a canonical model with states being maximal consistent sets. The key property of the canonical model is normally captured by the ``induction'' or ``truth'' lemma that ordinarily states that a formula is satisfied at a state if and only if it belongs to the corresponding maximal consistent set. In our case, satisfiability is defined as a relation $(a,s)\Vdash \phi$ between an agent $a$, a state $s$, and a formula $\phi$. As a result, in our construction, a maximal consistent set corresponds not to a state, but to a pair $(a,s)$ consisting of an agent $a$ and a state $s$. We informally refer to such pairs as ``views''. The induction lemma in our paper 
states that a formula is satisfied at a view if and only if it belongs to the maximal consistent set corresponding to this view.

There are three distinct challenges that we faced while proving the completeness theorem. The first of them is how to define agents and states, assuming that views are maximal consistent sets of formulae. Our first attempt was based on observation that two views that have the same states satisfy exactly the same $\A$-formulae. Thus, one can define states as classes of views (maximal consistent sets) that have the same $\A$-formulae. Similarly, it is reasonable to assume that if two sets have exactly the same $\K$-formulae, then they correspond to two views of the same agent in two indistinguishable states. Hence, one can define agents as classes of views that have the same $\K$-formulae. The problem with this approach is that there could be two distinct maximal consistent sets that have the same $\A$-formulae and the same $\K$-formula. Such sets could be unequal because, for example, one of them contains a propositional variable and the other the negation of the same variable. Informally, such sets would correspond to two different views of the same agent in the same state. This is problematic because our formal semantics captured in Definition~\ref{sat} assumes that if an agent $a$ is present in a state $s$, then she has a unique view $(a,s)$ in this state.

\begin{figure}
\begin{center}
\vspace{0mm}
\scalebox{0.45}{\includegraphics{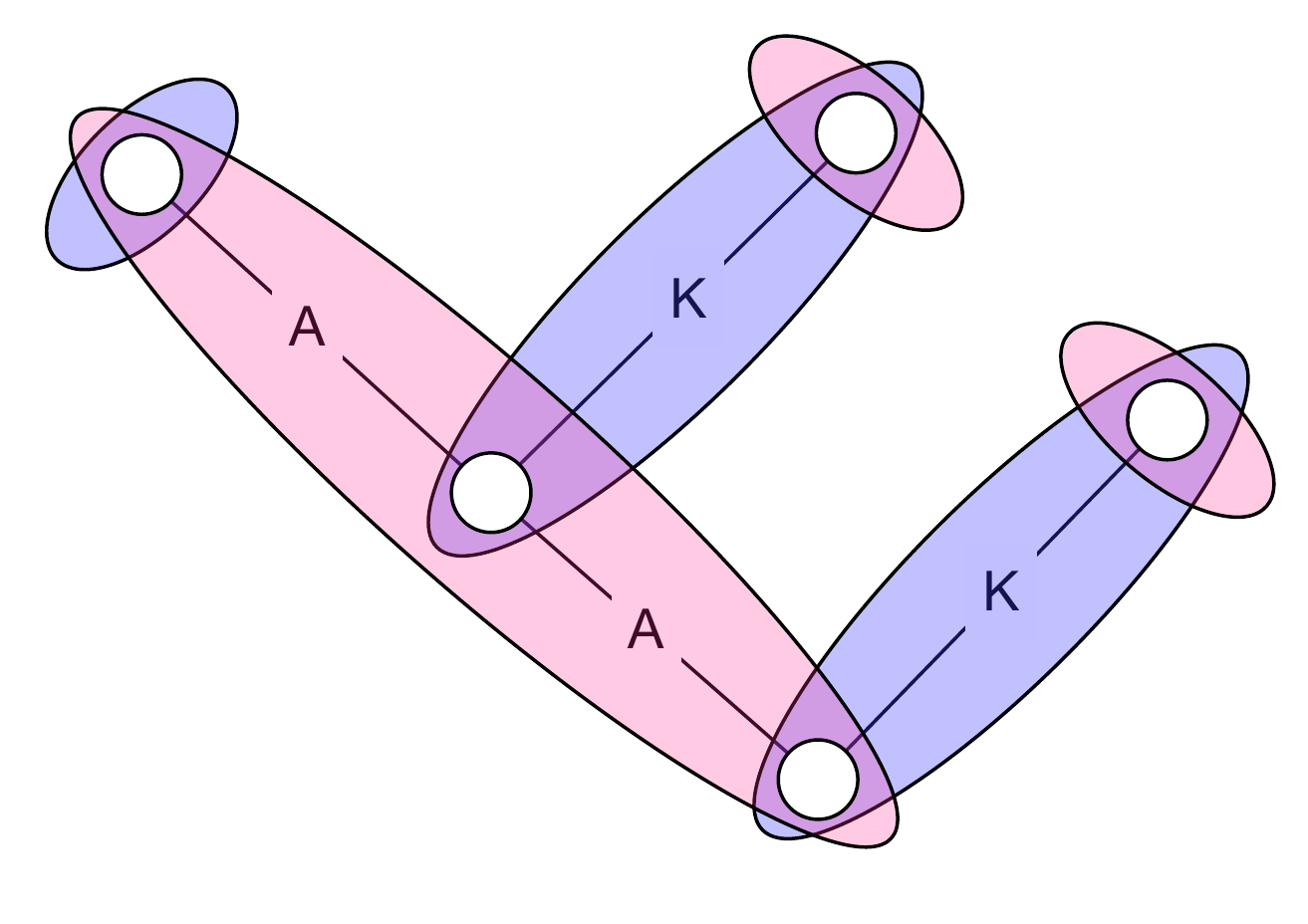}}
\caption{Nodes are views, pink $\A$-classes are states, and blue $\K$-classes are agents.}\label{tree with classes figure}
\end{center}
\end{figure}

To solve this problem, we need to guarantee that any class of views representing a state has at most one common element with any class of views representing an agent. We achieve this by using a {\em tree construction}. The canonical model in our proof is a tree whose nodes are labeled with maximal consistent sets and edges are labeled with a single modality: either $\A$ or $\K$, see Figure~\ref{tree with classes figure}. Informally, nodes of this tree correspond to views. We say that two nodes are $\A$-equivalent if all edges along the simple path between these two nodes are labeled with modality $\A$ and define states as equivalence classes with respect to this relation. Similarly, nodes are $\K$-equivalent if all edges along the simple path between them  are labeled with modality $\K$. Agents are $\K$-equivalence classes of nodes. Note that there is a unique simple path between any two nodes in a tree. As a result, the same two nodes cannot be $\A$-equivalent and $\K$-equivalent at the same time. Thus, this construction results in at most one node (view) corresponding to any pair consisting of an agent and a state. This guarantees that there is at most one view for any agent in any state. Of course, an agent ($\K$-equivalence class) might have no common nodes with a state ($\A$-equivalence class). In this case, the agent is not present in the state. 

As pointed out earlier, any two views that have the same state must have the same $\A$-formulae. We guarantee this by requiring any two nodes connected by an $\A$-edge to have the same $\A$-formulae. Similarly, we require any two nodes connected by a $\K$-edge to have the same $\K$-formulae.

Trees have previously been used in work on coalition know-how ~\cite{nt17aamas,nt18ai,nt18aaai,nt18aamas,cn20ai}, but for a different purpose -- to model distributed knowledge. The use of trees to guarantee that intersections of classes of nodes have at most one element is an original contribution of this work. 

The second major challenge that we had to overcome while proving the completeness is creating the actual nodes, or maximal consistent sets of formulae. The standard proof of completeness in modal logic usually contains a ``child'' lemma that for each maximal consistent set $X$ and each formula $\neg\Box\phi\in X$ constructs another set that contains formula $\neg\phi$. 
%
%
The situation is more complicated for modality $\W$ because one needs to construct two new interdependent maximal consistent sets simultaneously: one that corresponds to view $(a,s')$ and another to view $(a',s')$, see item 6 of Definition~\ref{sat}. Unfortunately, because of the interdependency, these two sets cannot be constructed consecutively. To construct them simultaneously, we developed a new technique that consists in defining a property of a pair of sets of formulae, choosing a pair of small sets satisfying this property, and then extending the sets while maintaining the property.  When fully extended, each of the sets will become the label of a node in the tree construction that we described above and will represent a view in our model. Informally, the property that we maintain could be described as ``views can co-exist in the same states''. We call such views {\em consonant}. A somewhat similar construction of two interdependent nodes has been used in~\cite{nt18ai,nt18aamas} to construct two states of a game in ``harmony''. The construction proposed in this paper creates two nodes that belong to the same state and, thus have the same $\A$-formulae. The two states in ``harmony'' are consecutive states of a game that do not share any specific class of formulae. As a result, the properties of consonant pairs are different from properties of pairs in ``harmony'' and the proofs that the corresponding constructions work are also different.   

The third challenge in constructing the canonical model is to define the right identification mechanism. What name should one of the blue classes (agents) in Figure~\ref{tree with classes figure} use to refer to another blue class in one of the pink classes (states)? The solution that we propose at first sounds unbelievably simple. In essence, when spokesman Birchim knows who is the killer, we want phrase ``is the killer'' to be the name under which Birchim knows the killer. In general if $(a,s)\Vdash \W\phi$, then formula $\phi$ itself is the name under which agent $a$ knows the agent with property $\phi$ in state $s$. Although elegant, this naming scheme has a fatal flaw: it does not distinguish between knowing that an agent {\em exists} and {\em knowing who} the agent is. For example, using this identification mechanism, spokesman Birchim would know who is the killer (agent named ``is the killer'') the moment Birchim is notified that the murder is committed. Similarly, a journalist arriving to the press conference would know who is the killer even before the conference starts.  In general, this naming space makes formula $\K\neg\A\neg\phi\to\W\phi$ true in any model that uses this identification mechanism. Since this formulae is not universally valid, the mechanism cannot be used in the canonical model construction of the completeness proof. 

We solve this problem by modifying the above identification mechanism. We still allow ``is the killer'' as the name, but we say that, when used by spokesman Birchim, this name refers to the actual killer only if in the current state Birchman actually knows who the killer is. Otherwise, when used by him, this name refers to all agents present in the state. Thus, if Birchman knows who the killer is, then he can use name ``is the killer'' to identify the killer, otherwise, he cannot. We are now ready to answer our prior question regarding names used by blue classes (agent) at pink classes (states) in Figure~\ref{tree with classes figure}. If the maximal consistent set of unique node at the intersection of an agent $a$ and a state $s$ contains formula $\W\phi$, then name $\phi$, when used by agent $a$ at state $s$, refers to all agents $b$ present in the state $s$ such that the maximal consistent set of the unique node at the intersection of agent $b$ and state $s$ contains formula $\phi$. Otherwise, name $\phi$ refers to all agents present in state $s$.  

\begin{figure}
\begin{center}
\vspace{0mm}
\scalebox{0.5}{\includegraphics{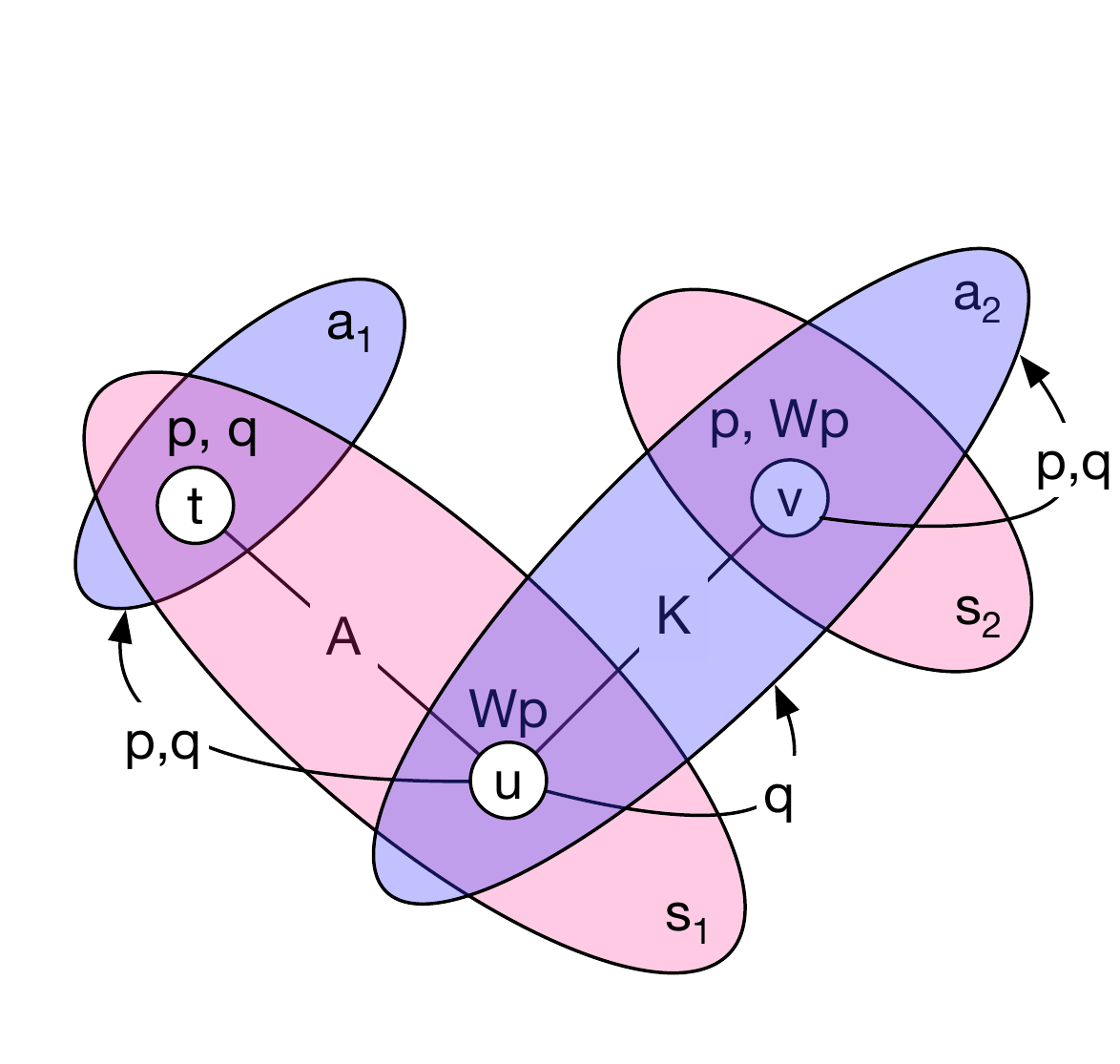}}
\caption{Fragment of a Canonical Model.}\label{tree with classes second figure}
\end{center}
\end{figure}

As an example, consider the fragment of the tree depicted in Figure~\ref{tree with classes second figure}. Nodes $u$ and $t$ are connected by an $\A$-edge. Thus, they represent views of two different agents, $a_1$ and $a_2$, in the same state $s_1$. On the other hand, nodes $u$ and $v$ are connected by a $\K$-edge. Hence, they represent two views of the same agent $a_2$ in two different (but indistinguishable to the agent) states: $s_1$ and $s_2$. Note that agent $a_1$ is not present in state $s_2$ because the corresponding ovals have no common nodes. The maximal consistent sets associated with views $u$ and $v$ contain formula $\W p$. As a result, when name $p$ is used in these two views, it refers to the agents in the same state whose maximal consistent sets contain variable $p$. In other words, when name $p$ is used by agent $a_2$ in state $s_1$, it refers to agent $a_1$ and when the same name is used by the same agent in state $s_2$, it refers to agent $a_2$ herself. At the same time, because formula $\W q$ does not belong to the maximal consistent sets corresponding to nodes $u$ and $v$, when name $q$ is used in these two views, it refers to all agents present in the corresponding state. In other words, in state $s_1$ name $q$ is used by agent $a_2$ to refer to herself and agent $a_1$; in state $s_2$ the same name is used by the same agent to refer only to herself.

This concludes the overview of the proof of the strong completeness theorem stated below. The complete proof can be found in the full paper.

\begin{theorem}
If $X\nvdash\phi$, then there is an agent $a\in A$ and a state $s\in P(a)$ of a model $(S,A,P,\{\sim_a\}_{a\in A},N,I,\pi)$ such that $(a,s)\Vdash\chi$ for each formula $\chi\in X$ and $(a,s)\nVdash\phi$. 
\end{theorem}

\section{Conclusion}\label{Conclusion section}

The contribution of this paper is three-fold. First, we proposed a formal semantics of modality {\em know-who} which is based on Grove-Halpern epistemic models with names.  Second, following~\cite{slg13tark} we propose a syntax for this modality that does not require the use of agent subscript. Without this modification to the language it would be hard to express statements like ``a journalist knows who knows who the murderer is''. Finally, we give a complete logical system that describes the interplay between modalities ``know-who", ``know'', and ``for all agents''. We believe that the standard filtration technique from modal logic could be used to prove weak completeness of our logical system with respect to the class of finite models. This would imply that our system, unlike logics with quantifiers previously used to capture know-who, is decidable. 
We also believe that the results in this paper could be generalized to a logical system that supports multiple name spaces. Each such name space $\eta$ will have its own identification mechanism $I_\eta$ in Definition~\ref{model} and its own modality $\W_\eta$. 


\bibliography{sp}

\end{document}